\documentclass[12pt]{article}

\usepackage{soul} 
\usepackage{geometry} 

\usepackage{graphicx} 

\usepackage{amsmath, amsfonts, amsthm, amssymb}

\usepackage{url} 

\usepackage[colorlinks=true, allcolors=blue]{hyperref} 

\usepackage[nameinlink]{cleveref} 

\theoremstyle{definition}

\newtheorem{theorem}{Theorem}
\newtheorem{definition}[theorem]{Definition}
\newtheorem{lemma}[theorem]{Lemma}
\newtheorem{corollary}[theorem]{Corollary}
\newtheorem{proposition}[theorem]{Proposition}
\newtheorem{problem}[theorem]{Problem}

\newtheorem{example}[theorem]{Example}
\newtheorem{remark}[theorem]{Remark}

\usepackage{enumitem} 

\makeatletter
\def\saveenum{\xdef\@savedenum{\the\c@enumi\relax}}
\def\resetenum{\global\c@enumi\@savedenum}
\makeatother

\usepackage[dvipsnames]{xcolor} 

\definecolor{darkblue}{rgb}{0.0, 0.0, 0.55}

\definecolor{cobalt}{rgb}{0.0, 0.28, 0.67}

\newcounter{todo}

\makeatletter

\newcommand\listtodoname{List of todos}
\newcommand\listoftodos{%
  \section*{\listtodoname}\@starttoc{tod}}
\makeatother
\DeclareMathOperator{\ActThr}{ActThr}
\DeclareMathOperator{\defect}{defect}
\DeclareMathOperator{\Fun}{Fun}

\newcommand{\btheta}{{\boldsymbol{\theta}}}

\newcommand{\bw}{\mathbf{w}}
\newcommand{\bfd}{\mathbf{d}}

\newcommand{\bz}{\mathbf{z}} 

\newcommand{\upperboundn}[1]{6{#1}^2-6#1+1}
\newcommand{\edimLHS}{d_L+\sum_{i=0}^{L-1}(d_id_{i+1}-d_{i+1})}
\newcommand{\edimRHS}{d_L\binom{r^{L-1}+d_0-1}{d_0-1}}

\usepackage{framed} 
\usepackage{subcaption} 
\usepackage{booktabs} 
\usepackage{arydshln} 
\usepackage{mathtools} 

\newcommand{\demph}[1]{\emph{{\color{RoyalBlue}#1}}}

\newcommand{\repeattheorem}[2]{%
  \begingroup
  \renewcommand{\thetheorem}{\ref{#1}}%
  \addtocounter{theorem}{-1}%
  \begin{theorem}
  #2
  \end{theorem}
  \endgroup
}

\makeatletter\let\saved@bibitem\@bibitem\makeatother
\usepackage{hyperref}
\makeatletter\let\@bibitem\saved@bibitem\makeatother

\usepackage{tikz} 
\newcommand{\qrw}[1]{q_{\mathbf{w}_{#1}}^r}
\newcommand{\qw}[1]{q_{\mathbf{w}_{#1}}}
\newcommand{\prw}[1]{p_{\mathbf{w}_{#1}}^r}
\newcommand{\pw}[1]{p_{\mathbf{w}_{#1}}}
\newcommand{\tildeprw}[1]{\tilde{p}_{\mathbf{w}_{#1}}^r}
\newcommand{\tildepw}[1]{\tilde{p}_{\mathbf{w}_{#1}}}

\renewcommand{\prw}[1]{Q_{({#1})}^r}
\renewcommand{\pw}[1]{Q_{({#1})}}
\renewcommand{\tildeprw}[1]{\tilde{Q}_{({#1})}^r}
\renewcommand{\tildepw}[1]{\tilde{Q}_{({#1})}}
\renewcommand{\qrw}[1]{Q_{({#1})}^r}
\renewcommand{\qw}[1]{Q_{({#1})}}
\newcommand{\wLone}[1]{w_{L_{1#1}}}
\newcommand{\tildewLone}[1]{w_{L_{1#1}}}
\renewcommand{\wLone}[1]{a_{({#1})}}
\renewcommand{\tildewLone}[1]{\tilde{a}_{({#1})}}

\DeclareMathOperator{\derivD}{d}

 \newcommand{\complexPsiMap}{\Psi_{{\mathbb{C},\mathbf{d},r}}} 
\newcommand{\sigmaCC}[1]{{\sigma_{#1}}}

\newcommand{\nNS}{n}
\newcommand{\kNS}{k}

\newcommand{\edim}[1]{\operatorname{edim}{#1}}
\renewcommand{\dim}[1]{\operatorname{dim}{#1}}

\title{Activation degree thresholds and expressiveness of polynomial neural~networks} 

\author{Bella Finkel\footnote{This material is based
upon work supported by the National Science Foundation Graduate Research Fellowship Program under Grant No.
2137424. Any opinions, findings, and conclusions or recommendations expressed in this material are those of the
authors and do not necessarily reflect the views of the National Science Foundation. Support was also provided by
the Graduate School and the Office of the Vice Chancellor for Research at the University of Wisconsin-Madison
with funding from the Wisconsin Alumni Research Foundation.}, Jose Israel Rodriguez\footnote{This research was partially supported by the Alfred P. Sloan Foundation.}, Chenxi Wu, Thomas Yahl}

\begin{document}

\maketitle

\begin{abstract}
We study the expressive power of deep polynomial neural networks through the geometry of their neurovariety. We introduce the notion of the activation degree threshold of a network architecture to express when the dimension of the neurovariety achieves its theoretical maximum. We prove the existence of the activation degree threshold for all polynomial neural networks without width-one bottlenecks and demonstrate a universal upper bound that is quadratic in the width of largest size. In doing so, we prove the high activation degree conjecture of Kileel, Trager, and Bruna. 
Certain structured architectures have exceptional activation degree thresholds, making them especially expressive in the sense of their neurovariety dimension.
In this direction, we prove that polynomial neural networks with equi-width~architectures are maximally expressive by showing their activation degree threshold is one.
\end{abstract}

\section*{Introduction}

Characterizing the functional space of artificial neural networks is a fundamental project in deep learning theory. When the functional space is semialgebraic, its geometric properties offer insight and mathematical intuition into the learning dynamics. Such a characterization is a two-way street: it is foundational to explaining empirical phenomena and offers the possibility of designing innovative architectures based on desirable theoretical and practical properties encoded in the space. Semialgebraic machine learning models form a rich class that encompasses, for example, Rectified Linear Unit (ReLU) and polynomial neural networks.

The study of deep polynomial neural networks offers a distinct advantage in theoretical machine learning because the functional spaces of such networks are described by algebraic varieties. In particular, a polynomial neural network of fixed architecture and activation degree gives an algebraic map from the network's weights to a set of polynomials. The image of this map is the space of functions that can be represented by the network. Its closure is an affine variety known as a neurovariety. The algebro-geometric invariants of a neurovariety, such as its dimension, degree, and singularities, encode information about the network's training dynamics and expressiveness. For a detailed discussion on neuroalgebraic geometry detailing the relationship between invariants of a neurovariety and certain fundamental aspects of machine learning, see \cite{marchetti2025invitationneuroalgebraicgeometry}. 

The most basic instantiation of a polynomial neural network occurs when the activation function is linear. Linear networks have received wide attention as simplified models for understanding the learning behavior of deep neural networks \cite{kawaguchi2016deeplearningpoorlocal, arora2019convergenceanalysisgradientdescent}.  Even in the case of linear networks, when the neuromanifold is a determinantal variety, there is interesting geometry to explore \cite{trager2020purespuriouscriticalpoints,kohn2024function,kohn2025geometrylinearneuralnetworks}. New developments in the general polynomial neural network setting have demonstrated the cross-disciplinary connections between algebraic geometry and machine learning and offered results on the viability of polynomial neural networks for applications \cite{KTB2019,shahverdi2024geometryoptimizationpolynomialconvolutional}. Results on the choice of the activation degree and the dimension of the neurovariety have improved our understanding of the optimization process of these neural networks and the ability of shallow and deep neural networks to replicate target functions \cite{KLW2024,Nguyen2018,arjevani2025geometryoptimizationshallowpolynomial}.

While polynomial neural networks with fixed activation degree are not universal approximators \cite{HORNIK1989359}, they become so when the activation degree is allowed to vary. An alternative measure of a polynomial neural network's expressiveness is the dimension of its neurovariety, which can be defined as the maximum number of independent vectors in the tangent space at a generic point. From the perspective of algebraic geometry,  the dimension of an algebraic variety is perhaps its most basic invariant. Characterizations of the neurovariety associated with a neural network represent a distinct description of a network's expressiveness from universal approximation theorems. The dimension of the neurovariety is a simple and precise measure of the degrees of freedom of the model that admits exact computation. 
In contrast to universal approximation theorems, the dimension of the neurovariety is a measure of the expressiveness of a given polynomial neural network. Thus, studying polynomial network expressiveness through the dimension of the neurovariety complements universal approximation results. For an explicit toy example illustrating neurovariety dimension, see \cite[Section A]{marchetti2025invitationneuroalgebraicgeometry}.

In practice, neural networks with polynomial output have been found to perform well in high-impact fields such as healthcare and finance, especially where the nature of the data is polynomial. Such applications include the prediction of epidemic behavior \cite{Fong2022}, quantification of the natural frequency of materials \cite{Dey}, analysis of financial time series data \cite{Ghazali,Nayak2018EstimatingSC} and the improvement of three-dimensional shape representation for computer vision \cite{Yavartanoo2021polynet}. 
In these applications, by introducing higher-order interactions between inputs, polynomials model non-linear phenomena more efficiently.

Other applications \cite{Jaffali_2020} take advantage of the fact that a homogeneous polynomial of degree $d$ in $n$ variables can be identified with a symmetric tensor in $\left(\mathbb{R}^n\right)^{\otimes d}$ that collects its coefficients. This description is also useful for obtaining theoretical results. It is exhaustive for shallow networks, where the function space coincides with a set of symmetric tensors of bounded rank \cite{KTB2019}. Moreover, it provides a lens for studying the loss surface of neural networks from an optimization perspective \cite{Venturi, arjevani2025geometryoptimizationshallowpolynomial}.

Substantial literature has shown the expressive power of deep networks from the perspective of the number of linear regions \cite{Montufar2014}, universal approximation and VC dimensions \cite{Telgarsky2016, Raghu2017}, a measure of complexity called ``trajectory length'' \cite{Mhaskar2016}, and the exact class of functions representable by ReLU networks of various depths \cite{Hertrich2023}. In another direction, \cite{Nguyen2018} studied the effect of the width of convolutional neural networks on the loss landscape, and \cite{Lu2017, Hanin2018} provided limits on the network width for universal approximation for ReLU networks, while \cite{KidgerLyons2020} provided such a limit for arbitrary activation functions. The expressiveness of neural networks has also been studied via their topological complexity, for example by \cite{bianchini2014complexity, GLMW2022}. In particular, \cite{Nguyen2018b} showed that networks with non-increasing architecture result in a connected decision region. ReLU expressiveness results have also been studied using algebraic geometry, specifically tropical geometry \cite{Brandenburg2024,GLMW2022, Hertrich2023}. 

An intuitive way of characterizing the expressiveness of a neural network is by the dimension of the space of representable functions as one allows its parameters to vary. For sigmoid activation functions this has been done in \cite{albertini1993neural}, for ReLU activation and non-increasing architectures, as well as a few others, this was done by \cite{bui2020functional, GLMW2022}. For polynomial activation functions, \cite{KTB2019, KLW2024} studied this quantity and raised various conjectures, one of which we will resolve in the present paper.

This paper is structured as follows. 
In \Cref{s:notation}, we review key definitions for polynomial neural networks and introduce the notion of \demph{activation degree thresholds} (\Cref{def:adt}).
In \Cref{s:high-degree}, we prove the existence of activation degree thresholds and 
determine the fibers of the parameter map of polynomial neural networks with high activation degree, resolving 
 \cite[Conjecture 16]{KTB2019} and \cite[Conjecture 5.2]{KLW2024}, raised by \cite{EleniusThesis}.
In \Cref{s:equi-width-setting}, we prove the expressiveness of networks with equi-width architecture by computing the activation degree threshold--- in particular, we remove any hypothesis involving  ``sufficiently high learning degree" by showing the activation degree threshold of these networks is one. 
In \Cref{s:future} we provide an outlook for the future, focusing on research directions opened by activation degree thresholds. 

\section{Polynomial neural networks~and~neurovarieties}\label{s:notation}

{\subsection{Background}}

An $L$-layer \demph{feedforward neural network} $F_\btheta:\mathbb{R}^{d_0}\to\mathbb{R}^{d_L}$ is a composition of affine-linear maps $f_i:\mathbb{R}^{d_{i-1}}\to\mathbb{R}^{d_i}$ and non-linear maps $\sigma_i:\mathbb{R}^{d_i}\to\mathbb{R}^{d_i}$, 
\begin{equation*}
    F_\btheta(x)=(f_L\circ\sigma_{L-1}\circ f_{L-1}\circ\cdots\circ f_2\circ\sigma_1\circ f_1)(x).
\end{equation*}
The \demph{architecture} of the neural network $F_\btheta$ is the sequence $\mathbf{d} = (d_0,d_1,\dotsc,d_L)$. 
Here, the affine-linear maps are given by
\begin{equation*}
    f_i:\mathbb{R}^{d_{i-1}}\to\mathbb{R}^{d_i},\quad x\mapsto W_i x + b_i,
\end{equation*}
where the \demph{weights} of the neural network are the matrices $W_1,W_2,\dots,W_L$ and the \demph{biases} are the vectors $b_1,\dotsc,b_L$. The parameter set $\btheta$ consists of these weight matrices and biases.
The \demph{activation map} $\sigma_i:\mathbb{R}^{d_i}\to\mathbb{R}^{d_i}$ is given coordinate-wise by the \demph{activation function}.

Deep learning consists of approximating a target function $F:\mathbb{R}^{d_0}\to\mathbb{R}^{d_L}$ by a neural network $F_\btheta:\mathbb{R}^{d_0}\to\mathbb{R}^{d_L}$ of a chosen architecture $\mathbf{d} = (d_0,d_1,\dotsc,d_L)$ and activation function. That is, deep learning concerns recovering parameters $\btheta$ for which $F_\btheta$ most closely resembles $F$. 
If $\mathbb{R}^N$ is the ambient space of the parameter set, the map $F_\btheta$ is associated to the $N$-tuple of parameters $\btheta$ by the \demph{parameter map}
\begin{equation}\label{eq:parameter-map}
    \Psi_{\mathbf{d},\sigma}:\mathbb{R}^{N}\to\Fun(\mathbb{R}^{d_0},\mathbb{R}^{d_L}),\quad \btheta\mapsto F_\btheta.
\end{equation}

In this article, we are interested in the setting in which the activation function is polynomial. The most important cases are when the activation function is a pure power, and this leads to the following definition. 

\begin{definition}[Polynomial neural network]\label{def:pnn}
    A \demph{polynomial neural network} $p_\textbf{w}:\mathbb{R}^{d_0}\to\mathbb{R}^{d_L}$ with fixed \demph{activation degree} $r$ and architecture $\mathbf{d}=(d_0,d_1,\dots,d_L)$ is a feedforward neural network of the form
    \begin{equation}\label{eq:pnn-function}
        p_\bw(x)=(W_L\circ\sigma_{L-1}\circ W_{L-1}\circ\sigma_{L-2}\circ\cdots\circ\sigma_1\circ W_1)(x)
    \end{equation}
    where $W_i\in\mathbb{R}^{d_i\times d_{i-1}}$ and the \demph{activation maps} $\sigma_i(x):\mathbb{R}^{d_i}\to\mathbb{R}^{d_i}$ are given by coordinate-wise exponentiation to the $r$-th power,
    \begin{equation*}
        \sigma_i(x):=(x_1^r,\dots,x_{d_i}^r).
    \end{equation*} The parameters $\bw$ are the entries in the matrices $W_i$, so that 
    \[
    \bw=(W_1,W_2,\dots,W_L).
    \]
\end{definition}

\begin{remark}
    In the definition of a polynomial neural network, the biases are taken to be zero. The affine-linear map $f_i$ in the $i$-th layer is then a linear map that we identify with the matrix $W_i$.
\end{remark}

An $L$-layer polynomial neural network $p_\bw$ with architecture $\mathbf{d} = (d_0,d_1,\dotsc,d_L)$ and activation degree $r$ is represented by a tuple of homogeneous polynomials. 
Precisely, the parameter map now takes the matrix tuple $(W_1,W_2,\dots, W_L)$ to a tuple of degree $r^{L-1}$ homogeneous polynomials in $d_0$ variables of length~$d_L$
\begin{align}\label{eq:pnn-parameter-map}
\begin{split}
    \Psi_{\mathbf{d},r}:\mathbb{R}^{d_1\times d_0}\times&\cdots\times\mathbb{R}^{d_L\times d_{L-1}}\to(\text{Sym}_{r^{L-1}}(\mathbb{R}^{d_0}))^{d_L},\\
    \bw&\mapsto p_\bw.
\end{split}
\end{align}
To specify an element in the image of $\Psi_{\mathbf{d},r}$, it suffices to identify its vector of coefficients in $\mathbb{R}^{d_L \binom{r^{L-1}+d_0-1}{d_0-1}}\simeq\left(\text{Sym}_{r^{L-1}}(\mathbb{R}^{d_0})\right)^{d_L}$.
 \begin{remark}[Special case]
     Note that if $d_L=1$, the codomain of the parameter map \eqref{eq:pnn-parameter-map} is the space of degree $r^{L-1}$ homogeneous polynomials. 
 \end{remark}

The \textit{neuromanifold} or \textit{functional space} associated to $p_{\mathbf{w}}$ is the family of all polynomial neural networks with fixed activation degree $r$ and architecture $\mathbf{d}$ 
where the weights are allowed to vary. The image of $\Psi_{\mathbf{d},r}$ is the \demph{neuromanifold} $\mathcal{M}_{\mathbf{d},r}$, which is a semialgebraic set in $(\text{Sym}_{r^{L-1}}(\mathbb{R}^{d_0}))^{d_L}$. (In spite of the terminology, the neuromanifold is not in general a smooth manifold, and its singularities can affect the behavior of the corresponding model, see \cite{Amari,marchetti2025invitationneuroalgebraicgeometry, shahverdi2024geometryoptimizationpolynomialconvolutional}.) The \demph{neurovariety} $\mathcal{V}_{\mathbf{d},r}$ is the Zariski closure of $\mathcal{M}_{\mathbf{d},r}$, and is a real affine variety in $(\text{Sym}_{r^{L-1}}(\mathbb{R}^{d_0}))^{d_L}$.

The dimension of the neurovariety $\mathcal{V}_{\mathbf{d},r}$ provides a measure of the expressiveness of a neural network because it quantifies the degrees of freedom of the space of functions the network produces \cite{KTB2019}. In practice, this is the dimension of the neuromanifold $\mathcal{M}_{\mathbf{d},r}$; as $\mathcal{M}_{\mathbf{d},r}$ is semialgebraic, its dimension is the same as that of $\mathcal{V}_{\mathbf{d},r}$.

The neurovariety $\mathcal{V}_{\bfd,r}$  is a real variety parameterized by the polynomial map 
$\Psi_{\mathbf{d},r}:\mathbb{R}^{N}\to\mathbb{R}^M$
where $N=\sum_{i=0}^{L-1}d_id_{i+1}$ and $M=\edimRHS$. 
Extending the domain of $\Psi_{\mathbf{d},r}$ to $\mathbb{C}^N$
results in the map  
\begin{equation}\label{eq:cc-psi}
\complexPsiMap:\mathbb{C}^N\to\mathbb{C}^M.
\end{equation}
Denote the Zariski closure of the image of $\complexPsiMap$ by $Y$. 
By \cite[Theorem 12.2]{Sottile-Real-AG-Handbook}, 
the dimension of the complex variety $Y$ and the dimension of the real variety $\mathcal{V}_{\bfd,r}$ agree if $Y$ contains a smooth real point. 
Since $Y$ is parameterized by polynomials with real coefficients, it follows that $Y$ contains a smooth real point and 
\begin{equation}\label{eq:real-dim-equal}
    \dim{Y} 
    =
    \dim{\mathcal{V}_{\bfd,r}}.
\end{equation}
Equation \eqref{eq:real-dim-equal} puts complex algebraic geometry at the disposal of statements concerning the dimension of $\mathcal{V}_{\bfd,r}$. We exploit this in \Cref{theorem:high-degree,thm:equi-width-expected-v2}. 
Our approach relies on 
understanding the dimension of the \demph{fiber} of $\complexPsiMap$ over 
$y\in \complexPsiMap(\mathbb{C}^N)$: 
\[\complexPsiMap^{-1}(y):=\{ w\in \mathbb{C}^N: \complexPsiMap(w)=y\}.\] 

For all network architectures there exists a symmetry in the weight matrices, known as \demph{multi-homogeneity}.

\begin{lemma}[\cite{KTB2019}, Multi-homogeneity]\label{lem:multi-hom}
    For all invertible diagonal matrices $D_i\in\mathbb{R}^{d_i\times d_i}$ and permutation matrices $P_i\in\mathbb{Z}^{d_i\times d_i}$ ($i=1,\dots,L-1$), the parameter map $\Psi_{\mathbf{d},r}$ returns the same neural network under the replacement
    \begin{align*}
        W_1&\leftarrow P_1D_1W_1\\
        W_2&\leftarrow P_2D_2W_2D_1^{-r}P_1^T\\
        \vdots\\
        W_L&\leftarrow W_L D_{L-1}^{-r} P_{L-1}^T,
    \end{align*}
    where $T$ denotes the matrix transpose. 
\end{lemma}

By letting the diagonal matrices $D_i$ have complex entries, 
the multi-homogeneity lemma has known consequences for the fibers of $\complexPsiMap$.
Namely, the fiber of
$\complexPsiMap$
over  $y=\complexPsiMap(W_1,\dots,W_L)$ satisfies this containment:
\begin{equation}\label{eq:complex-fiber}
\begin{split}
    \complexPsiMap^{-1}(y)\supseteq
    \Big\{
(P_1D_1W_1,\,
P_2D_2W_2D_1^{-r}P_1^T,\,
\dots, \,
W_L D_{L-1}^{-r} P_{L-1}^T) 
:
\,\phantom{xxxxxxxxxxxx}&\\
P_i \text{ is a permutation matrix and }D_i\text{ is an invertible diagonal matrix}
\Big\},
\end{split}
\end{equation}
and the dimension of a general fiber is at least $\sum_{i=1}^{L-1}d_i$ \cite[Lemma 13]{KTB2019}. For clarity, we remark that when $L=1$, the right-hand side of  \eqref{eq:complex-fiber}  simplifies to $\{W_1\}$, and 
the containment becomes an equality.
We are careful to state when a polynomial map is over a  real or complex domain to avoid ambiguity when discussing the dimension of preimages that are not general.

\medskip
Considering the difference of the size of the weight matrices and the dimension of the space of multi-homogeneities leads to the notion of the expected dimension of a neurovariety.

\begin{definition}[Expected dimension]\label{def:expected-dim}
The \demph{expected dimension} of the neurovariety $\mathcal{V}_{\mathbf{d},r}$ is
\begin{align}\label{eq:edim-numbers}
    \edim{\mathcal{V}_{\mathbf{d},r}}:=\min&\Bigg\{\edimLHS, \quad \edimRHS\Bigg\}.
\end{align}
\end{definition}

This definition was introduced in \cite{KLW2024} and in \cite{KTB2019}
it was shown that $\edim{\mathcal{V}_{\mathbf{d},r}}$ is an upper bound on the dimension of the neurovariety.
\begin{example}\label{ex:L-equals-one}
In the case when $L=1$, the neurovariety $\mathcal{V}_{\mathbf{d},r}$ has the expected dimension $d_0d_1$, as the map $\Psi_{\mathbf{d},r}$ is an isomorphism. 
The function space is parameterized by the $d_1\times d_0$ matrix $W_1$.
\end{example}

The expected dimension is not always equal to the dimension, as the following example demonstrates.

\begin{example}
    The dimension of a neurovariety is subtle even in  simple settings. 
    For instance, when the architecture is $(d_0,d_1,1)$, a reinterpretation~\cite[Theorem 9]{KTB2019} of the  
    {Alexander--Hirschowitz} Theorem~\cite{AH1995} determines which neurovarieties have the expected dimension.
    In \cite{KLW2024}, 
    the authors coin the term \emph{defect} for the difference between the expected dimension of a neurovariety and its dimension:
    {
\[
\defect(\mathcal{V}_{\bfd,r}):=
\edim{\mathcal{V}_{\bfd,r}}
-
\dim{\mathcal{V}_{\bfd,r}}.
\]  
    }%
Examples of neurovarieties
    with positive defect are found in 
    \cite[Table 1]{KLW2024} and in
    \cite{KTB2019} as well. 
\end{example}

\begin{remark}
    The notion of neurovariety dimension as a measure of network expressiveness has been utilized for designing neural networks. For example, \cite{Jaffali_2020} employs a hybrid network with monomial and LeakyReLU activation functions to determine the entanglement type of quantum states. Because their architecture is shallow, they apply the {Alexander--Hirschowitz} Theorem when choosing the number of neurons in each layer to guarantee that the associated neurovariety achieves the expected dimension.
Such applications motivate the project of describing the neurovariety dimension for broader classes of architectures. Our \Cref{theorem:high-degree} and \Cref{thm:equi-width-expected-v2} are examples of this kind of result.
\end{remark}

\subsection{Main result and example}
One of our main results proves a lower bound on the dimension of a neurovariety for sufficiently high activation degree. 
We introduce the notion of the \demph{activation degree threshold}, whose existence will follow from our \Cref{theorem:high-degree}.

\begin{problem}
\label{prob:threshold}
    Given an architecture $\mathbf{d}$,  
    does there exist a nonnegative integer $\tilde{r}$
    such that the following holds:
    \begin{equation}\label{eq:ActThr}
    \dim{\mathcal{V}_{\mathbf{d},r}}=\edim{\mathcal{V}_{\mathbf{d},r}}\quad \text{for all}\quad r>\tilde{r}\,?
    \end{equation}
\end{problem}

When such a nonnegative integer $\tilde r$ exists for \eqref{eq:ActThr} we have the following definition.

\begin{definition}\label{def:adt}
The \demph{activation degree threshold} of an architecture $\mathbf{d} = (d_0,d_1,\dots,d_L)$ with $L>1$
    is 
    \[
\ActThr(\mathbf{d}):=\min\{ \tilde r\in \mathbb{N}_{\geq 0} : 
\dim{\mathcal{V}_{\mathbf{d},r}}=\edim{\mathcal{V}_{\mathbf{d},r}} \text{ for all } r>\tilde{r}
\}.    
    \]
    In other words, 
    $\ActThr(\mathbf{d})$ is
    the smallest {nonnegative integer} $\tilde{r}$ such that \eqref{eq:ActThr} holds.
\end{definition}

With this terminology in hand, we state our main result, which shows that the activation degree threshold is well-defined when $\bfd$ has no entries equal to one. 

\repeattheorem{theorem:high-degree}{
For fixed $\mathbf{d} = (d_0,d_1,\dotsc,d_L)$ satisfying $d_i>1$ ($i=0,\dotsc,L-1$),
the activation degree threshold $\ActThr(\mathbf{d})$ exists. 
In other words, there exists an integer $\tilde{r}(\bfd)$ such that whenever $r>\tilde{r}(\bfd)$, the neurovariety $\mathcal{V}_{\mathbf{d},r}$ of the polynomial neural network has the expected dimension, 
\begin{align*}
\dim{\mathcal{V}_{\mathbf{d},r}} = \edim{\mathcal{V}_{\mathbf{d},r}} = d_L + \sum_{i=0}^{L-1} (d_i-1)d_{i+1}.
\end{align*}
Moreover,
for $L>1$,
we have 
\begin{equation}
\ActThr(\mathbf{d}) \le 
\upperboundn{m},
\quad m:=-1+2\max\{d_1, \dots, d_{L-1}\} .
\end{equation}
}

\begin{example}\label{ex:three-two-one}
This example aims to illustrate the activation degree threshold by mapping the function space of neural networks with architecture $\mathbf{d} = (3,2,1)$ and activation function $x^r$ for $r\in \mathbb{R}_{>0}$. To ensure these functions are well-defined, we restrict $\mathbb{R}^{d_0}$ to the nonnegative orthant and assume that each weight matrix has nonnegative entries.

Denote the corresponding parameter space as $\Theta^+$ and 
the neuromanifold as $\Fun^+_r\left(\mathbb{R}^{3}_{\geq 0},\mathbb{R}_{\geq0}\right)$. 
The neuromanifold 
$\Fun^+_r\left(\mathbb{R}^{3}_{\geq 0},\mathbb{R}_{\geq0}\right)$ 
has dimension at most six for all {$r\in \mathbb{R}_{>0}$}  
because any element in it is of the form 
\[
(ax_1+bx_2+cx_3)^r+(a'x_1+b'x_2+c'x_3)^r.
\]

Note that each point $\bz$ in $\mathbb{R}^{3}$ induces an evaluation map  
\[
\Fun_r^+\left(\mathbb{R}^{3}_{\geq 0},\mathbb{R}_{\geq 0}\right)\to\mathbb{R}_{\geq 0},
\quad 
 F\mapsto F(\bz).
\]
Therefore, consider six points $\bz_1,\dots,\bz_6\in\mathbb{R}^{3}$ and define the function \[
E: \Fun_r^+(\mathbb{R}_{\geq0}^3,\mathbb{R}_{\geq0})
\to \mathbb{R}_{\geq0}^6\] 
by
{$F\mapsto (F({\bz_1}),\dots,F({\bz_6}))$.}
Precomposing with the parameter map~\eqref{eq:parameter-map},
we get a family of functions depending on $r$:
\begin{equation}\label{eq:positive-map}
\mathbb{R}^{2\times 3}\times \mathbb{R}^{1\times 2}{\supset\Theta^+}\to
E\left(
\Fun_r^+(\mathbb{R}_{\geq0}^3,\mathbb{R}_{\geq0})\right)
\subset
\mathbb{R}^6,\quad
\btheta=(W_1,W_2)\mapsto 
E\left(F_\btheta\right).
\end{equation}

The benefit of this setup is that it
allows for a straightforward rank computation of the Jacobian matrix.
For example, we calculate the Jacobian matrix of the map \eqref{eq:positive-map} with the following choices.  
The six points we choose to define $E$
are 
\[(1, 0, 0), \,\, (0, 1, 0),\,\, (0, 0, 1),\,\, (1, 1, 0),\,\, (1, 0, 1),\,\, (0, 1, 1),\] 
and  we  calculate the Jacobian of \eqref{eq:positive-map}
at the point
\[\left(\left[\begin{array}{ccc}1 & 1 & 1\\1& 2 & 3\end{array}\right], \left[\begin{array}{cc}1 & 1\end{array}\right]\right)\in \mathbb{R}^{2\times 3}\times \mathbb{R}^{1\times 2}.\]
The resulting 
Jacobian matrix is 
\begin{equation}\label{eq:six-eight-jacobian}
\left[\begin{array}{cccccc;{4pt/3pt}cc}
r&0&0&r&0&0 &1&1\\
0&r&0&0&r2^{r-1}&0 &1&2^r\\
0&0&r&0&0&r3^{r-1} &1&3^r\\
r2^{r-1}&r2^{r-1}&0&r3^{r-1}&r3^{r-1}&0 &2^r&3^r\\
r2^{r-1}&0&r2^{r-1}&r4^{r-1}&0&r4^{r-1} &2^r&4^r\\
0&r2^{r-1}&r2^{r-1}&0&r5^{r-1}&r5^{r-1} &2^r&5^r\\
\end{array}\right].
\end{equation}

Denote the $6\times 6 $ submatrix on the left by $J(r)$.
The positive integers where the $J(r)$ drops rank are $r=1,2$. Therefore the activation degree threshold of this architecture is less than or equal to two. 
When $r=2$, $F_\btheta$ is the sum of two squares, hence a degenerate quadratic form in $3$ variables. Therefore it has at most $5$ degrees of freedom, and so the threshold is exactly two.

We visualize the functions $\det(J(r))$ and $\log|\det(J(r))|$   in Figure~\ref{fig:det-log-det-six-eight}. Explicitly, the value of $\det(J(r))$ is the following polynomial exponential:
    \begin{equation}\label{eq:det-six-six}
    r^6 \left(-96^{r-1}-32^{r-1}+3\cdot 48^{r-1}+80^{r-1}-36^{r-1}-3\cdot 60^{r-1}+90^{r-1}+30^{r-1}\right).
        \end{equation}

\begin{figure}[hbt!]
    \centering
    \begin{tikzpicture}
        \draw[-](0,0)--(5, 0);
        \draw[-](6,0)--(11,0);
        \draw[-](1, 1)--(1, -4);
        \draw[-](7, 1)--(7, -4);
        \draw[dashed, orange](2, 1)--(2, -4);
        \draw[dashed, orange](3, 1)--(3, -4);
        \draw[dashed, orange](8, 1)--(8, -4);
        \draw[dashed, orange](9, 1)--(9, -4);
        \draw[blue, -](0.00,0.000003)--(0.05,0.000002)--(0.10,0.000002)--(0.15,0.000001)--(0.20,0.000001)--(0.25,0.000001)--(0.30,0.000000)--(0.35,0.000000)--(0.40,0.000000)--(0.45,0.000000)--(0.50,0.000000)--(0.55,0.000000)--(0.60,0.000000)--(0.65,0.000000)--(0.70,0.000000)--(0.75,0.000000)--(0.80,0.000000)--(0.85,0.000000)--(0.90,0.000000)--(0.95,0.000000)--(1.00,0.000000)--(1.05,-0.000000)--(1.10,-0.000000)--(1.15,-0.000000)--(1.20,-0.000000)--(1.25,-0.000000)--(1.30,-0.000000)--(1.35,-0.000000)--(1.40,-0.000000)--(1.45,-0.000000)--(1.50,-0.000000)--(1.55,-0.000000)--(1.60,-0.000000)--(1.65,-0.000000)--(1.70,-0.000000)--(1.75,-0.000000)--(1.80,-0.000000)--(1.85,-0.000000)--(1.90,-0.000000)--(1.95,-0.000000)--(2.00,0.000000)--(2.05,0.000000)--(2.10,0.000001)--(2.15,0.000005)--(2.20,0.000017)--(2.25,0.000050)--(2.30,0.000129)--(2.35,0.000303)--(2.40,0.000657)--(2.45,0.001338)--(2.50,0.002584)--(2.55,0.004755)--(2.60,0.008370)--(2.65,0.014107)--(2.70,0.022734)--(2.75,0.034867)--(2.80,0.050369)--(2.85,0.067060)--(2.90,0.078146)--(2.95,0.067346)--(3.00,0.000000)--(3.05,-0.192694)--(3.10,-0.641024)--(3.15,-1.583255)--(3.20,-3.443451);
        \draw[red, -](6.00,-1.275304)--(6.02,-1.285270)--(6.04,-1.295561)--(6.06,-1.306192)--(6.08,-1.317174)--(6.10,-1.328524)--(6.12,-1.340257)--(6.14,-1.352389)--(6.16,-1.364938)--(6.18,-1.377923)--(6.20,-1.391365)--(6.22,-1.405286)--(6.24,-1.419709)--(6.26,-1.434661)--(6.28,-1.450167)--(6.30,-1.466259)--(6.32,-1.482968)--(6.34,-1.500331)--(6.36,-1.518384)--(6.38,-1.537171)--(6.40,-1.556737)--(6.42,-1.577133)--(6.44,-1.598415)--(6.46,-1.620644)--(6.48,-1.643890)--(6.50,-1.668230)--(6.52,-1.693749)--(6.54,-1.720545)--(6.56,-1.748727)--(6.58,-1.778421)--(6.60,-1.809769)--(6.62,-1.842936)--(6.64,-1.878114)--(6.66,-1.915527)--(6.68,-1.955439)--(6.70,-1.998166)--(6.72,-2.044089)--(6.74,-2.093672)--(6.76,-2.147492)--(6.78,-2.206275)--(6.80,-2.270955)--(6.82,-2.342761)--(6.84,-2.423356)--(6.86,-2.515071)--(6.88,-2.621320)--(6.90,-2.747392)--(6.92,-2.902147)--(6.94,-3.102190)--(6.96,-3.384796)--(6.98,-3.868899)--(7.00,-4.000000)--(7.02,-3.867076)--(7.04,-3.381160)--(7.06,-3.096762)--(7.08,-2.894956)--(7.10,-2.738480)--(7.12,-2.610737)--(7.14,-2.502878)--(7.16,-2.409626)--(7.18,-2.327579)--(7.20,-2.254417)--(7.22,-2.188489)--(7.24,-2.128579)--(7.26,-2.073767)--(7.28,-2.023341)--(7.30,-1.976739)--(7.32,-1.933513)--(7.34,-1.893300)--(7.36,-1.855802)--(7.38,-1.820774)--(7.40,-1.788015)--(7.42,-1.757356)--(7.44,-1.728658)--(7.46,-1.701807)--(7.48,-1.676707)--(7.50,-1.653284)--(7.52,-1.631479)--(7.54,-1.611249)--(7.56,-1.592565)--(7.58,-1.575413)--(7.60,-1.559794)--(7.62,-1.545723)--(7.64,-1.533233)--(7.66,-1.522375)--(7.68,-1.513221)--(7.70,-1.505869)--(7.72,-1.500448)--(7.74,-1.497125)--(7.76,-1.496118)--(7.78,-1.497711)--(7.80,-1.502280)--(7.82,-1.510329)--(7.84,-1.522548)--(7.86,-1.539918)--(7.88,-1.563882)--(7.90,-1.596690)--(7.92,-1.642122)--(7.94,-1.707279)--(7.96,-1.808123)--(7.98,-1.995612)--(8.00,-4.000000)--(8.02,-1.955678)--(8.04,-1.728236)--(8.06,-1.587399)--(8.08,-1.482192)--(8.10,-1.396632)--(8.12,-1.323598)--(8.14,-1.259290)--(8.16,-1.201436)--(8.18,-1.148570)--(8.20,-1.099690)--(8.22,-1.054082)--(8.24,-1.011216)--(8.26,-0.970690)--(8.28,-0.932194)--(8.30,-0.895479)--(8.32,-0.860347)--(8.34,-0.826635)--(8.36,-0.794210)--(8.38,-0.762961)--(8.40,-0.732793)--(8.42,-0.703630)--(8.44,-0.675404)--(8.46,-0.648060)--(8.48,-0.621552)--(8.50,-0.595839)--(8.52,-0.570891)--(8.54,-0.546680)--(8.56,-0.523188)--(8.58,-0.500399)--(8.60,-0.478306)--(8.62,-0.456904)--(8.64,-0.436198)--(8.66,-0.416197)--(8.68,-0.396919)--(8.70,-0.378390)--(8.72,-0.360648)--(8.74,-0.343742)--(8.76,-0.327740)--(8.78,-0.312732)--(8.80,-0.298839)--(8.82,-0.286223)--(8.84,-0.275112)--(8.86,-0.265828)--(8.88,-0.258847)--(8.90,-0.254917)--(8.92,-0.255294)--(8.94,-0.262341)--(8.96,-0.281376)--(8.98,-0.329381)--(9.00,-4.000000)--(9.02,-0.287341)--(9.04,-0.197284)--(9.06,-0.136172)--(9.08,-0.087010)--(9.10,-0.044469)--(9.12,-0.006171)--(9.14,0.029152)--(9.16,0.062261)--(9.18,0.093648)--(9.20,0.123647)--(9.22,0.152502)--(9.24,0.180391)--(9.26,0.207449)--(9.28,0.233785)--(9.30,0.259484)--(9.32,0.284614)--(9.34,0.309233)--(9.36,0.333388)--(9.38,0.357120)--(9.40,0.380463)--(9.42,0.403447)--(9.44,0.426098)--(9.46,0.448438)--(9.48,0.470488)--(9.50,0.492264)--(9.52,0.513784)--(9.54,0.535060)--(9.56,0.556106)--(9.58,0.576933);
    \end{tikzpicture}
    \caption{ The blue curve is the {graph} of 
    $r\mapsto \det(J(r))$ where $J(r)$ is the $6\times 6$ submatrix of \eqref{eq:six-eight-jacobian}. 
    The red curve is $(r,\log|y|)$ where $(r,y)$ is on the blue curve. The orange dashed lines are $r=1$ and $r=2$.
}
    \label{fig:det-log-det-six-eight}
\end{figure}

\end{example}

\begin{remark}
For some intuition, we relate the activation degree threshold to $\defect(\mathcal{V}_{\bfd,r})$. 
For fixed $\mathbf{d} = (d_0,d_1,\dotsc,d_L)$ satisfying $d_i>1$ {($i=0,\dotsc,L-1$)},
\[\defect(\mathcal{V}_{\bfd,r})=0,
\text{ for all }r>\ActThr(\bfd),
\]
and $\ActThr(\bfd)=0$ or 
$\defect(\mathcal{V}_{\bfd,\ActThr(\bfd)})$ is positive.
\end{remark}

\section{Expressiveness for high activation degree}\label{s:high-degree}

In this section, we develop results on the linear independence of powers of polynomials (\Cref{ss:number-theory} and \Cref{ss:hidden-layer-high-degree}) to prove our main result in \Cref{ss:deep-high-degree}.
In particular, our \Cref{theorem:high-degree} completes the results in \Cref{corollary:universal-bound}. Our \Cref{corollary:universal-bound} completes the proof of \cite[Theorem 14]{KTB2019}.

\subsection{A prior number theoretic result by Newman--Slater}\label{ss:number-theory}

In this section, we recover a number-theoretic result stated by
{Newman--Slater} in \cite{NS1979}.
 The proof here slightly improves their bound {(which is $n\leq 8\kNS^2$ in our notation)}.

\begin{theorem}\label{theorem:NS-rational}
Let $n>1$ and \[R_1^\nNS(x) + R_2^\nNS(x) + \cdots + R_\kNS^\nNS(x) = 1\] where  
$R_1,R_2,\dots,R_k$ are 
{non-constant rational functions} 
in the variable $x$.
Then $\nNS\le \upperboundn{\kNS}$. 
\end{theorem}
\begin{proof}
A proof of this statement follows from modifications of the proof of \cite[Section 4, Theorem]{NS1979}. We include it here for the benefit of the reader. 

\smallskip

Consider the $\kNS \times \kNS$ Wronskian matrix $W:=W(R_1^\nNS,\dots,R_\kNS^\nNS)$ where the $(i,j)$-th entry is 
the $(i-1)$-st derivative of $R_j^\nNS$, which we denote by $\derivD^{i-1} R_j^\nNS$.
With this setup, we have the following vector equation, with the first standard basis vector on the right-hand side and $\mathbf{1}$ denoting the vector of ones:
\[W
\cdot
\mathbf{1}=\mathbf{e}_1.\]

By letting $D$ be the $\kNS\times \kNS$ diagonal matrix with $(j,j)$ entry 
$1/{R_j^{\nNS+1-\kNS}}$ we have
\[\left(W\cdot D\right) \left(D^{-1}\cdot \mathbf{1}\right)=\mathbf{e}_1\]
where 
\[
W\cdot D =
\begin{bmatrix}
    \dfrac{1}{R_1^{1-\kNS}} &\dots & \dfrac{1}{R_\kNS^{1-\kNS}}\\  
    \vdots  & & \vdots\\
    \dfrac{ \derivD^{\kNS-1} R_1^\nNS}{R_1^{\nNS+1-\kNS}} 
    &\dots &\dfrac{ \derivD^{\kNS-1} R_\kNS^\nNS}{R_\kNS^{\nNS+1-\kNS}}\\
\end{bmatrix}.
\]

The benefit of this formulation is that cancellation of common factors in numerators and denominators occurs with the multiplication  $W\cdot D$.

If the rational functions $R_1^n,\dotsc,R_k^n$ are linearly independent, then  $\det(W)\ne 0$ and hence $\det(W\cdot D)\ne 0$. 
Using Cramer's rule, we solve for the first entry of 
$\left(D^{-1}\cdot \mathbf{1}\right)$:
\begin{equation}\label{eq:cramer}
    R_1^{\nNS+1-\kNS}=\det(B)/\det(W \cdot D),
\end{equation}
where $B$ is obtained by replacing the first column of $W\cdot D$ with  $\mathbf{e}_1$.

Assume each $R_j$ is written in lowest terms. After reordering the $R_j$'s as necessary, suppose $R_1$ has the largest degree numerator or denominator and denote this degree by $A$. Hence, the sum of the degrees of the numerator and denominator on the left-hand side of  \eqref{eq:cramer} is between 
$A\cdot |\nNS-\kNS +1|$ and $2A\cdot |\nNS-\kNS+1|$.

Now note that if $f$, $g$ are polynomials, then
\[(f/g^r)'=(f'g^r-frg'g^{r-1})/g^{2r}=(f'g-frg')/g^{r+1}.\]
In other words, the sum of the degrees of the numerator and denominator, when taking the derivative once, increases by no more than $2\deg(g)-1$. 
As a consequence, 
the $(i, j)$-th entry of the matrix $W\cdot D$ when written in lowest terms is a rational function for which the sum of the degrees of the numerator and denominator is no more than $2A(\kNS-1)+(2A-1)(i-1)$.
By the Leibniz formula for the determinant, the sum of the degrees of the numerator and denominator for $\det(B)$ is no more than 
\begin{align*}
\sum_{i=2}^\kNS\big(2A\cdot(\kNS-1)+(2A-1)(i-1)\big)=&2A\cdot(\kNS-1)^2\,+\,(2A-1)\kNS(\kNS-1)/2\\=&A\cdot(\kNS-1)(3\kNS-2)\,-\,\kNS(\kNS-1)/2\end{align*}
while the sum of the degrees of the numerator and denominator for $\det(W\cdot D)$ is no more than
\begin{align*}
\sum_{i=1}^\kNS\big(2A\cdot(\kNS-1)+(2A-1)(i-1)\big)=&2A\cdot\kNS(\kNS-1)+(2A-1)\kNS(\kNS-1)/2\\=&A\cdot(\kNS-1)(3\kNS)-\kNS(\kNS-1)/2.\
\end{align*}
Using \eqref{eq:cramer},
we have
\begin{align*}
A\cdot(\nNS-\kNS{+1})\\
    &\hspace{-.5in}\leq 
        A\cdot|\nNS-\kNS+1|
        \\
    &\hspace{-.5in}\leq
        \Big(
            A\cdot(\kNS-1)(3\kNS-2)-\kNS(\kNS-1)/2
        \Big)+
        \Big(
            A\cdot(\kNS-1)(3\kNS)-\kNS(\kNS-1)/2
        \Big).
\end{align*}

This implies the inequality
\begin{align*}
n\le \left(6-\frac{1}{A}\right)k^2 + \left(-7+\frac{1}{A}\right)k + 1.
\end{align*}
As $A\ge 1$, we arrive at 
\begin{equation}\label{eq:result-linearly-independent}
n\le \upperboundn{\kNS}.
\end{equation}

We have shown the inequality \eqref{eq:result-linearly-independent}
holds when $R_1^n,\dots,R_\kNS^n$ are linearly independent. If these functions are linearly dependent, then we may find a subset $S$ of $\{1,\dots,\kNS\}$ with cardinality $\ell$ less than $\kNS$ such that 
$\sum_{j\in S} R_j^n=1$ {(after potentially rescaling the $R_j^n$)}
and $\{R_j^n : j\in S\}$ is linearly independent.
Applying the  Wronskian argument in this setting proves
\[
    n\le \upperboundn{\ell}. 
\]
Since $1<\ell<\kNS$ implies
 $\upperboundn{\ell}< \upperboundn{\kNS}$, the theorem holds.
\end{proof}

\begin{corollary}\label{cor:NS-polynomials}
Let $P_1,\dots,P_\kNS$ and $R$ be polynomials in the variable $x$ such that $R$ is not the zero polynomial and no $P_i$ is proportional to $R$.  
If $P_1^\nNS + P_2^\nNS + \cdots + P_\kNS^\nNS = R^\nNS$, then 
$n\le \upperboundn{\kNS}$.

\end{corollary}
\begin{proof}
Since $R$ is not the zero polynomial, 
the equation $P_1^\nNS + P_2^\nNS + \cdots + P_\kNS^\nNS = R^\nNS$
holds if and only if 

    \[\left(\frac{P_1}{R}\right)^\nNS + \left(\frac{P_2}{R}\right)^\nNS+ \cdots +
    \left(\frac{P_\kNS}{R}\right)^\nNS
    =1.\]
The result follows from  \Cref{theorem:NS-rational} since each $P_i/R$ is a non-constant rational function. 
\end{proof}

\subsection{Powers of non-proportional {multivariate} polynomials}\label{ss:hidden-layer-high-degree}

The statements in the previous subsection are for univariate polynomial functions. In this subsection, we extend those results to the multivariate polynomial setting to apply them to polynomial neural networks in \Cref{ss:deep-high-degree}. 

To that end, first we generalize \Cref{cor:NS-polynomials} 
to multivariate polynomials.

\begin{lemma}\label{lemma:ns-generalization}
Let $p_1,\dotsc,p_k\in\mathbb{C}[x_1,\dotsc,x_d]$ denote multivariate polynomials that are pairwise non-proportional (for any $i\not=j$, there is no $\alpha\in \mathbb{C}$ such that $p_i=\alpha p_j$). 
If
\begin{align*}
p_1^r + \dotsb + p_k^r = 0,
\end{align*}
then 
$r\le \upperboundn{(\kNS-1)}$.
\end{lemma}

\begin{proof}
Let $p_1,\dots, p_k$ be pairwise non-proportional polynomials on $\mathbb{C}^d$ such that
\[
p_1^r+\cdots+p_k^r =0
\]
holds identically.  
Consider the line $\ell(t)=tx+(1-t)y$ in $\mathbb{C}^d$ that  passes through the points $x$ and $y$ and is parameterized by $t$. 
Restricting each $p_i$ to this line yields the univariate polynomials 
$p_1(\ell(t)),\dots,p_k(\ell(t))$ in $t$.
Then
    \begin{align*}
        \big(p_1(\ell(t))\big)^r + \dotsb + \big(p_k(\ell(t))\big)^r = 0
    \end{align*}
holds identically in $t$. 

If  $x,y\in\mathbb{C}^d$ are general points,
one has
    \begin{align*}
    p_i(x)/p_j(x)
    \,\ne\, p_i(y)/p_j(y),\quad {i\neq j},
    \end{align*}
which implies 
    \begin{align*}
    p_i(\ell(1))/p_j(\ell(1))
    \,\ne\,
    p_i(\ell(0))/p_j(\ell(0)),\quad i\neq j.
    \end{align*}   
Therefore the univariate polynomials 
 $p_1(\ell(t)),\dots,p_k(\ell(t))$ are pairwise non-proportional, as any quotient of a pair is non-constant.
  By \Cref{cor:NS-polynomials},
    we conclude that 
\[r\le \upperboundn{(\kNS-1)}.\]
\end{proof}

Next we demonstrate that given pairwise non-proportional multivariate polynomials $p_1,\dotsc,p_k$, the polynomials $p_1^r,\dotsc,p_k^r$ are linearly independent for all sufficiently large $r$. 

\begin{proposition}
\label{proposition:linear-independence}
Let $\mathbb{K}$ be a subfield of $\mathbb{C}$. 
Given integers $d, k$, there exists an integer $\tilde{r}=\tilde{r}(k)$ with the following property. 
If $r>\tilde{r}(k)$ and $p_1,\dots,p_k\in \mathbb{K}[x_1,\dots,x_d]$ are pairwise non-proportional, then $p_1^r,\dots,p_k^r$ are linearly independent (over $\mathbb{K}$). 
Moreover, 
{$\tilde{r}(k) = \upperboundn{(\kNS-1)}$} has the desired property.
\end{proposition}

\begin{proof}
    We remark that it suffices to consider the case that $\mathbb{K} = \mathbb{C}$. Indeed, if $p_i$ and $p_j$ are not linearly dependent over $\mathbb{K}$, then they are not linearly dependent over $\mathbb{C}$. For if there were a constant $\alpha\in\mathbb{C}$ such that $p_i - \alpha p_j = 0$, then one finds that $\alpha\in\mathbb{K}$ by considering the coefficients of this difference. Further, linear independence of $\{p_1^r,\dotsc,p_k^r\}$ over $\mathbb{C}$ implies linear independence over any subfield. 
    
    We prove the contrapositive statement when $\mathbb{K} = \mathbb{C}$. Given integers $d,k$, let
    $\tilde{r}(k) = \upperboundn{(k-1)}$ and $r>\tilde{r}(k)$ be an integer. Fix pairwise non-proportional polynomials $p_1,\dotsc,p_k\in\mathbb{C}[x_1,\dotsc,x_d]$ with the property that the set $\{p_1^r,\dotsc,p_k^r\}$ is linearly dependent over $\mathbb{C}$. Then there exist $\alpha_1,\dots,\alpha_k\in \mathbb{C}$ such that 
    \begin{align*}
    \alpha_1p_1^r+\cdots +\alpha_k p_k^r=0.
    \end{align*}
    Let $\beta_i$ be an $r$th root of $\alpha_i$, so that
    \begin{align*} 
    (\beta_1p_1)^r+\cdots +(\beta_k p_k)^r=0.
    \end{align*}
    By Lemma~\ref{lemma:ns-generalization}, it follows that for some $i\ne j$, $\beta_i p_i$ and $\beta_j p_j$ are proportional. Thus, some $p_i$ and $p_j$ are proportional.    
\end{proof}

\begin{corollary}[Conjecture 16, \cite{KTB2019}]\label{corollary:universal-bound}
There exists $\tilde r:\mathbb{N}^2_{>0}\to\mathbb{Z}$, $(d,k)\mapsto\tilde r(d,k)$ with the following property. Whenever $p_1,...,p_k\in \mathbb{R}[x_1,...,x_d]$ are $k$ homogeneous polynomials of the same degree in $d$ variables, no two of which are linearly dependent, $p^r_1,\dots , p^r_k$ are linearly independent if $r>\tilde r({d,k})$.     
\end{corollary}
\begin{proof}
Note that $p_i,p_j$ being linearly independent implies $p_i,p_j$ are  non-proportional. 
    By \Cref{proposition:linear-independence} it follows that  
    $\tilde r(d,k)=\upperboundn{(\kNS-1)}$
    suffices. 
\end{proof}

\subsection{Deep networks and high activation degree}\label{ss:deep-high-degree}

Next, we show when activation degree thresholds exist for polynomial neural networks. The proof of the following theorem follows the arguments for \cite[Theorem 14]{KTB2019}, except where we leverage our \Cref{proposition:linear-independence}.

\begin{theorem}\label{theorem:high-degree}
For fixed $\mathbf{d} = (d_0,d_1,\dotsc,d_L)$ satisfying $d_i>1$ ($i=0,\dotsc,L-1$),
the activation degree threshold $\ActThr(\mathbf{d})$ exists. 
That is, there exists an integer $\tilde{r}(\bfd)$ such that whenever $r>\tilde{r}(\bfd)$, the neurovariety $\mathcal{V}_{\mathbf{d},r}$ of the polynomial neural network has the expected dimension, 
\begin{equation}\label{edim_at_threshold}
\dim{\mathcal{V}_{\mathbf{d},r}} = \edim{\mathcal{V}_{\mathbf{d},r}} = \edimLHS
\end{equation}
Moreover,
for $L>1$,
we have 
\begin{equation}\label{eq:universal-bound-in-theorem}
\ActThr(\mathbf{d}) \le 
\upperboundn{m},
\quad m:=-1+2\max\{d_1, \dots, d_{L-1}\} .
\end{equation}

\end{theorem}

\begin{proof}
Fix $r> \tilde{r}(\mathbf{d}) =  
\upperboundn{m}$.

Since the expected dimension of the neurovariety $\mathcal{V}_{\bfd,r}$ is the minimum of the two numbers in \eqref{eq:edim-numbers}, it suffices to prove $\dim{\mathcal{V}_{\bfd,r}}=\edimLHS$ for $r>\tilde{r}(\bfd)$.
{We show this 
by inducting} on the number of layers $L$ 
to prove a general fiber of
$\complexPsiMap$ in~\eqref{eq:cc-psi}
is equal to the set of multi-homogeneities of any point of the fiber. That is, we show the containment of sets in \eqref{eq:complex-fiber} is an equality.

Recall $\complexPsiMap$ from \eqref{eq:cc-psi}.
When $L=1$, the neurovariety $\mathcal{V}_{\mathbf{d},r}$ has the expected dimension as the map
$\complexPsiMap$
is an isomorphism 
and~$\complexPsiMap^{-1}(\complexPsiMap(W_1))=\{W_1\}$.

Now assume $L>1$ 
and suppose that  {a general} fiber for depth $L-1$ is completely described by the multi-homogeneity in \eqref{eq:complex-fiber}. 
 Fix general weights $(W_1,\dots,W_L)\in \mathbb{C}^{d_0 d_1}\times\cdots \times \mathbb{C}^{d_{L-1} d_L}$
and consider given weights $(\tilde{W}_1,\dots,\tilde{W}_L)\in \mathbb{C}^{d_0 d_1}\times\cdots \times \mathbb{C}^{d_{L-1}d_L}$
such that
\begin{align}\label{eq:KTB13}
\begin{split}
    &(W_L\sigma_{L-1} W_{L-1}\dots\sigma_1 W_1)(x)\\ &=(\tilde{W}_L\sigma_{L-1}\tilde{W}_{L-1}\dots\sigma_1 \tilde{W}_1)(x).
\end{split}
\end{align}
Denote the output of $$(W_{L-1}\sigma_{L-2} W_{L-2}\dotsb W_2\sigma_1 W_1)(x)$$ by 
    $\left[\pw{1}(x),\dots,\pw{d_{L-1}}(x)\right]^T$
and the output of $$(\tilde{W}_{L-1}\sigma_{L-2}\tilde{W}_{L-2}\dotsb \tilde{W}_2\sigma_1 \tilde{W}_1)(x)$$ 
by 
    $\left[\tildepw{1}(x),\dots,\tildepw{d_{L-1}}(x)\right]^T$.
Since the weights $W_i$ are general and $d_i>1$, the homogeneous polynomials 
$\pw{i}$
are pairwise non-proportional 
(there is a non-empty Zariski open set of the space of weight matrices where the 
polynomials
$\pw{1},\dots,\pw{d_{L-1}}$
are pairwise non-proportional). 

By examining the first coordinate of the outputs of \eqref{eq:KTB13} we obtain the decomposition
\begin{align}\label{eq:KTB14}
\begin{split}
    & {\wLone{1}\prw{1}
    + \wLone{2}\prw{2}
    + \dots
    + \wLone{d_{L-1}}\prw{d_{L-1}}}
    \\&= {
     \tildewLone{1}
     \tildeprw{1}
    +\tildewLone{2}
    \tildeprw{2}
    +\dots
    +\tildewLone{d_{L-1}} \tildeprw{d_{L-1}}}
\end{split}
\end{align}
where 
{$\wLone{j}$ and $\tildewLone{j}$ denote the $(1,j)$} entry of the weight matrices $W_L$ and $\tilde{W}_L$ respectively.
Because we have taken $r>\tilde{r}(\bfd)$,  \Cref{proposition:linear-independence} guarantees that \eqref{eq:KTB14} has two proportional summands. 
As $\pw{i}$
are pairwise non-proportional,
no two summands on the left-hand side may be proportional. 
The equality implies that neither may two summands on the right-hand side be proportional. 

By permuting terms as necessary, we may assume these proportional terms 
occur as the first term on both sides. 
We may then rescale so that 
$\pw{1}=\tildepw{1}$
and subtract $\tildewLone{1}\tildeprw{1}$
to obtain 
\begin{align}\label{eq:KTB15}
\begin{split}
    \left(\wLone{1}-\tildewLone{1}   \right)
    \prw{1}
    + \wLone{2}\prw{2}
    + \dots
    + \wLone{d_{L-1}}\prw{d_{L-1}}
    \\=  
    \tildewLone{2}  
    \tildeprw{2}
    +\dots
    +\tildewLone{d_{L-1}} \tildeprw{d_{L-1}}
\end{split}
\end{align}
We repeatedly apply \Cref{proposition:linear-independence} to iteratively reduce the number of terms on the right side, until the right side of \eqref{eq:KTB15} is zero. Then, the pairwise linear independence of the $\pw{i}$ 
and \Cref{proposition:linear-independence} imply that each coefficient is also zero. Hence, up to scaling and permutation, 
    $\left[\pw{1},\dots,\pw{d_{L-1}}\right]^T=\left[\tildepw{1},\dots,
    \tildepw{d_{L-1}}
    \right]^T$ 
and the entries of $W_L$ and $\tilde{W}_{L}$ agree in the first row. This process is repeated for each row of $W_L$ and $\tilde{W}_L$ to show that they are equal (up to scaling and permutation).
Therefore, by the inductive hypothesis the fiber $\complexPsiMap^{-1}(\complexPsiMap(W_1,W_2,\dots,W_L))$ 
 is completely described by multi-homogeneity.
 As a consequence, the fiber 
has dimension $(d_1+d_2+\cdots+d_{L-2})+d_{L-1}.$
By the 
Fiber-Dimension Theorem~\cite[Theorem~1.25, Part (ii)]{Shafarevich-Hirsch} and Equation \eqref{eq:real-dim-equal}, 
the dimension of $\mathcal{V}_{\mathbf{d},r}$ satisfies 
\[
d_{1}+d_2+\cdots +d_{L-1}=
\sum_{i=0}^{L-1}d_id_{i+1} - \dim{\mathcal{V}_{\bfd,r}},
\] and the result follows.
\end{proof}

\begin{corollary}[Conjecture 5.2, \cite{KLW2024}]
    For any fixed widths $\bfd = (d_0, \dots , d_L)$ 
    with $d_i > 1$ for $i = 1,\dots,L-1$, there exists an integer $\tilde r(\bfd)$ such that whenever $r > \tilde r(\bfd)$,  the neurovariety $\mathcal{V}_{\bfd,r}$ of a polynomial neural network attains the expected dimension.
\end{corollary}

Table 1 of \cite{KLW2024} provides several examples of neurovarieties that have the expected dimension. In this next section, we show that equi-width architectures have activation degree threshold one, that is, the corresponding neurovarieties always have the expected dimension when the activation degree is two or more. 

\begin{remark}
{In the proof of \Cref{theorem:high-degree}, we leveraged a surprising result about pairwise non-proportionality implying linear independence. 
Linear independence of powers of polynomials is part of a larger research program of understanding Fermat hypersurfaces and tickets~\cite{Reznick2001} of polynomial systems.}
\end{remark}

\section{Equi-width setting}\label{s:equi-width-setting}
The advantage of depth for neural networks is well-established through both empirical observations and theoretical results \cite{Montufar2014, Telgarsky2016, Mhaskar2016, Raghu2017}. Width also plays a significant role in network expressiveness and improves the behavior of the loss surface by smoothing the optimization landscape of the loss function \cite{Hertrich2023,Nguyen2018, Nguyen2018b}. The study of networks with bounded width appears in the context of universal approximation theorems \cite{Hanin2018,KidgerLyons2020,Lu2017} as well.

For sufficiently high activation degree, a neurovariety has the expected dimension by \Cref{theorem:high-degree}. For specific architectures, better bounds on the activation degree threshold may be computed.
In this section, we focus on determining the expressiveness of polynomial neural networks with \demph{equi-width} architecture, meaning the layers have equal 
widths, that is, 
$d_0=d_1=\dots=d_L$
in Definition~\ref{def:pnn}.

We begin with a proposition which will be used in the main result of this section.

\begin{proposition}\label{prop:zeros-V2}
Let 
{$\sigmaCC{i}:\mathbb{C}^{d}\to\mathbb{C}^{d}$} 
be the map given by coordinate-wise exponentiation to the $r$-th power. 
{Suppose $W_1,\dots,W_L$ are  in $\mathbb{C}^{d\times d}$. }
The map 
$p_{\mathbf{w}}:\mathbb{C}^d\to \mathbb{C}^d$ given by
\[p_{\mathbf{w}}(x) = (W_L\sigmaCC{L-1} W_{L-1}\dotsb W_2\sigmaCC{1} W_1)(x)\] 
has only the trivial zero if and only if all $W_i$ are invertible. 
\end{proposition}

\begin{proof}
If each $W_i$ is invertible, then the preimage of the zero vector under each $W_i$ and $\sigmaCC{i}$ consists of only the zero vector. 
Thus, the map 
$p_{\mathbf{w}}:\mathbb{C}^d\to \mathbb{C}^d$  
has only the trivial zero. 

Conversely, if some $W_i$ is singular, then let $i_0$ be the minimal index such that $W_{i_0}$ 
is singular and let $v\in\ker W_{i_0}$ be a non-zero vector. 
For each $j<i_0$, the linear map $W_j$ is surjective, and each 
$\sigmaCC{j}:\mathbb{C}^d\to\mathbb{C}^d$
is surjective as well.
Thus, there exists non-zero 
$x_\ast\in\mathbb{C}^d$ such that 
\begin{align*}
(\sigmaCC{i_0-1} W_{i_0-1}\dotsb W_2\sigmaCC{1} W_1)(x_\ast) = v. 
\end{align*}
As $v\in\ker W_{i_0}$, it follows that $p_{\mathbf{w}}(x_\ast) = (W_L\sigmaCC{L-1} W_{L-1}\dotsb W_2\sigmaCC{1} W_1)(x_\ast) = 0$ and $x_\ast$ is a non-trivial zero.
\end{proof}

We now provide the main result of this section: the activation degree threshold of an equi-width architecture is one.

\begin{theorem}\label{thm:equi-width-expected-v2}
If $\mathbf{d} = (d_0,d_1,\dotsc,d_L)$ is an equi-width architecture with $d = d_0 =d_1 \dotsb = d_L$, $L>1$, and $d>1$, then the activation degree threshold is $\ActThr(\mathbf{d}) = 1$. That is, if $\mathbf{d}$ is equi-width with $L>1$ and $d>1$, then for all $r>1$ the neurovariety $\mathcal{V}_{\mathbf{d},r}$ has the expected dimension 
\begin{align*}
\dim{\mathcal{V}_{\mathbf{d},r}} = Ld^2 - (L-1)d.
\end{align*}
\end{theorem}

\begin{proof}
If $r=1$, then $\mathcal{V}_{\mathbf{d},1}$ is the space of linear functions from $\mathbb{R}^{d}$ to $\mathbb{R}^d$. Therefore $\dim{\mathcal{V}_{\mathbf{d},1}} = d^2$. 
Since \begin{align*}
 \edim{\mathcal{V}_{\mathbf{d},1}} =  Ld^2 - (L-1)d,
\end{align*}
we have  $\defect(\mathcal{V}_{\mathbf{d},1})=(L-1)(d^2 - d)$. Therefore for all $L>1$, $\ActThr(\bfd)>0$.

\smallskip

Recall that the neurovariety $\mathcal{V}_{\bfd,r}$  is the Zariski closure of the image of the polynomial map 
$\Psi_{\mathbf{d},r}:\mathbb{R}^{N}\to\mathbb{R}^M$
where $N=Ld^2$ and $M={d \binom{r^{L-1}+d-1}{d-1}}$. 
Extending the domain of $\Psi_{\mathbf{d},r}$ to $\mathbb{C}^N$ 
results in the map \eqref{eq:cc-psi}
with the Zariski closure of the image denoted by $Y$. 
By Equation \eqref{eq:real-dim-equal},
to complete the proof, 
 it suffices to show $\dim{Y}=N-(L-1)d$, which we will do using the Fiber-Dimension Theorem~\cite[Theorem~1.25]{Shafarevich-Hirsch} for \eqref{eq:dimY-equals-edim}.

For notational convenience, we let 
$S_{L,d,r}$ denote the set of parameters that are mapped to $\complexPsiMap(I_d,\dots,I_d)$:
\begin{equation}\label{eq:S-fiber}
S_{L,d,r}:=
\left\{
(W_1,\dots,W_L)\in \left(\mathbb{C}^{d\times d}\right)^L:
\complexPsiMap(I_d,\dots,I_d)
=
\complexPsiMap(W_1,\dots,W_L)
\right\}.
\end{equation}
We want to show for $L>1$ and $d>1$ that
\begin{align}\label{eq:induction-step}
\begin{split}
S_{L,d,r}=
\Big\{
(W_1,\dots,\,
PDW_{L-1},
D^{-r} P^T)
\in \left(\mathbb{C}^{d\times d}\right)^L:
(W_1,\dots,W_{L-1})\in S_{L-1,d,r},&\\
\,P \text{ is a permutation matrix and }
D\text{ is an invertible diagonal matrix}
\Big\}.
\end{split}
\end{align}

Note that if  $(W_1,W_2,\dotsc,W_L)\in S_{L,d,r}$ then
\begin{align*}
(W_L\sigma_{L-1}W_{L-1}\dotsb W_2\sigma_1 W_1)(x) = \begin{pmatrix}
x_1^{r^{L-1}}\\
\vdots\\
x_d^{r^{L-1}}
\end{pmatrix}.
\end{align*}
Moreover, there is the equality

\begin{align}\label{eq:monoms-q}
W_L
\begin{pmatrix}
\qrw{1}(x)\\
\vdots\\
\qrw{d}(x)
\end{pmatrix}
=
\begin{pmatrix}
x_1^{r^{L-1}}\\
\vdots\\
x_d^{r^{L-1}}
\end{pmatrix}
\end{align}
where 
\begin{align*}
    \begin{pmatrix}
        \qw{1}(x)\\
        \vdots\\
        \qw{d}(x)
    \end{pmatrix}
    :=
        (W_{L-1}\sigma_{L-2}W_{L-2}\dotsb W_2\sigma_1 W_1)(x).
\end{align*}
Since the polynomial neural network $(x_1^{r^{L-1}},\dotsc,x_d^{r^{L-1}})$ only has the trivial zero, \Cref{prop:zeros-V2} implies that $W_L$ is invertible. Multiplying both sides of \eqref{eq:monoms-q} by $W_L^{-1}$, each $\qrw{i}$ 
is a linear combination of $x_1^{r^{L-1}},\dotsc,x_d^{r^{L-1}}$. For $r >1$, this is only possible if each 
$\qw{i}$
is a monomial. 
Indeed, no linear combination of $x_1^{r^{L-1}},\dotsc,x_d^{r^{L-1}}$ is a pure $r$-th power except for scalar multiples of these individual monomials. Thus each 
$\qrw{i}$
is a scalar multiple of some $x_j^{r^{L-1}}$ and each
$\qrw{i}$
is a scalar multiple of some $x_j^{r^{L-2}}$. 
It follows that $W_L$ is a scaled permutation matrix---it is an invertible matrix that sends a set of monomials to a set of scalar multiples of those monomials. 
Thus, we have shown \eqref{eq:induction-step}.

By induction, we get 
\begin{align*}
\begin{split}
S_{L,d,r}=
\Big\{
(P_1D_1,
P_2D_2D_1^{-r}P_1^T,
\dots,\,
D_{L-1}^{-r} P_{L-1}^T)
\in 
\left(\mathbb{C}^{d\times d}\right)^L:
\,\phantom{\text{ is a permutation matrix}}&\\
P_i \text{ is a permutation matrix and }D_i\text{ is an invertible diagonal matrix}
\Big\}.
\end{split}
\end{align*}

We have shown that 
the fiber $\complexPsiMap^{-1}(\complexPsiMap(I_d,\dotsc,I_d))$
is $S_{L,d,r}$ and therefore 
\[
\dim{
    \complexPsiMap^{-1}(\complexPsiMap(I_d,\dotsc,I_d))
    }
    \,=\,
\dim{
    S_{L,d,r}
    }
    \,=\,
    (L-1)d.
\]
It follows from the
 Fiber-Dimension Theorem
 \cite[Theorem~1.25, Part (i)]{Shafarevich-Hirsch} that
\begin{equation}\label{eq:dim-Y-bound}
    N - \dim{Y} 
    \,\leq\, (L-1)d.
\end{equation}
Using these facts on the expected dimension:  
\[
\edim{\mathcal{V}_{\mathbf{d},r}} \,=\,N-(L-1)d\text{ \quad and \quad  }\edim{\mathcal{V}_{\mathbf{d},r}}\geq \dim{\mathcal{V}_{\mathbf{d},r}},
\]
along with \eqref{eq:dim-Y-bound},
we have
\begin{equation}\label{eq:dimY-equals-edim}
    \edim{\mathcal{V}_{\mathbf{d},r}} \,\leq\, \dim{Y}\,=\,\dim{\mathcal{V}_{\mathbf{d},r}}\,\leq \,\edim{\mathcal{V}_{\mathbf{d},r}}.  
\end{equation}

Therefore $\mathcal{V}_{\mathbf{d},r}$ has the expected dimension for all $r>1$ and $\ActThr(\mathbf{d}) = 1$.
\end{proof}

\section{Outlook}\label{s:future}

\subsubsection*{Universal optimal bound for the activation degree threshold}

Our bound in \Cref{theorem:high-degree}  on the activation degree threshold for polynomial neural networks is not universally optimal. 
We proved this to be the case for equi-width networks in \Cref{thm:equi-width-expected-v2}.  
Improving the bound established in \Cref{lemma:ns-generalization} leads to better bounds on the activation degree threshold in \Cref{theorem:high-degree}.
One potential way to improve the universal bound of activation degree thresholds is to leverage other algebro-geometric and number theoretic results. 
Specifically, if the polynomials $p_1,\dotsc,p_k$ in \Cref{lemma:ns-generalization} satisfy the additional hypothesis of having no common zeros, then the result in \cite[page 71]{G1975} implies the improved bound $r\leq k^2- 1$. In the language of algebraic geometry, \cite[page 71]{G1975}  establishes a Picard theorem for holomorphic maps into Fermat hypersurfaces. 
We are hopeful 
that other results in this spirit will be applicable to the study of the dimension of neurovarieties.

\Cref{thm:equi-width-expected-v2} suggests that we should be looking for architectures where the activation degree threshold significantly improves upon the universal upper bound. A direct generalization of \Cref{thm:equi-width-expected-v2} in this direction is the extension to non-increasing architectures.

\subsubsection*{Pyramidal polynomial neural networks (Non-increasing architectures)}
One future direction in this research area is to determine the activation degree thresholds for  \demph{non-increasing} architectures $\mathbf{d} = (d_0,d_1,\dotsc,d_L)$ --- those architectures satisfying $d_0\ge d_1\ge\dotsb\ge d_L$. 
By generalizing the proof of Theorem~\ref{thm:equi-width-expected-v2}, one might hope to prove  \cite[Conjecture 5.7]{KLW2024} that for a non-increasing architecture $\mathbf{d}$ with $d_L>1$, the neurovariety $\mathcal{V}_{\mathbf{d},r}$ has the expected dimension for all $r>1$. That is, prove that a non-increasing architecture $\mathbf{d}$ has activation degree threshold less than or equal to~one.

\subsubsection*{Other networks}
Another research direction is to introduce the notion of activation degree thresholds for other neural networks. Permitting negative exponents in a polynomial activation map leads to the notion of rational neural networks \cite{boulle2020rational}. There is experimental evidence that it is beneficial to choose rational activation maps that are of very low degree. Defining activation degree thresholds for rational networks would provide a quantitative measure to make these observations rigorous. Rational neural networks are currently being investigated by Alexandros Grosdos et al. Certain piecewise linear networks may also admit the definition of an activation degree threshold as a means to study expressiveness. Tropical rational neural networks provide a lens to study networks with piecewise linear activation maps \cite{alfarra2022decision,Charisopoulos,charisopoulos2018tropical,zhang18}. Such networks encompass ReLU networks; a ReLU network of fixed architecture is described by a semialgebraic set inside the parameter space of tropical rational functions. In this context, tropical rational functions provide means to define activation degree thresholds to frame and obtain results on expressiveness. Moreover, because the dimension of the neurovariety is preserved under tropicalization, activation degree thresholds may aid in translating expressiveness results for tropical neural networks to polynomial networks. 

\subsubsection*{Acknowledgments}
{The authors would like to thank Jiayi Li, Kaie Kubjas, and Alexandros Grosdos for their comments and feedback that led to improvements of this article.}
We also thank the referees for their thoughtful suggestions and questions.
\bibliographystyle{siam-no-dash-title-color-links} 
\bibliography{refs} 

\end{document}